\documentclass{article} 
\usepackage{IJCNN2023_conference,times}

\usepackage{amsmath,amsfonts,bm}

\def\eqref#1{equation~\ref{#1}}

\def\1{\bm{1}}

\DeclareMathAlphabet{\mathsfit}{\encodingdefault}{\sfdefault}{m}{sl}
\SetMathAlphabet{\mathsfit}{bold}{\encodingdefault}{\sfdefault}{bx}{n}

\usepackage[utf8]{inputenc} 
\usepackage[T1]{fontenc}    
\usepackage{hyperref}       
\usepackage{url}            
\usepackage{booktabs}       
\usepackage{amsfonts}       
\usepackage{nicefrac}       
\usepackage{microtype}

\usepackage{url}
\usepackage{bm}
\usepackage{amsmath}
\usepackage{amsthm}
\usepackage{amssymb}
\usepackage{lipsum}
\usepackage{pifont}
\usepackage{graphicx}
\usepackage{subcaption}
\usepackage{caption}
\usepackage{varwidth}
\usepackage{centernot}
\usepackage{todonotes}
\usepackage{mathtools}
\usepackage{algorithm}
\usepackage[noend]{algpseudocode}

\usepackage{multicol}
\usepackage{multicol}
\usepackage{lipsum}
\usepackage[font=small,labelfont=bf]{caption}
\usepackage[normalem]{ulem} 
\usepackage{natbib}
\newtheorem{lemma}{Lemma}
\newtheorem{definition}{Definition}
\newtheorem{remark}{Remark}

\usepackage{float}
\usepackage{lscape}

\title{ \Large The Backpropagation algorithm for a math student}

\iclrfinalcopy

\author{\textbf{Saeed Damadi\textsuperscript{1}}, \textbf{Golnaz Moharrer\textsuperscript{2}}, \textbf{Mostafa Cham\textsuperscript{2}}, \textbf{Jinglai Shen\textsuperscript{1}} \\
  $^1 \text{Department of Mathematics and Statistics}$, $^2 \text{Department of Information Systems}$
  \\
  University of Maryland, Baltimore County (UMBC)\\
  Baltimore, MD 21250 \\
  \texttt{sdamadi1, golnazm1, mcham2, shenj@umbc.edu}
}

\begin{document}

\maketitle

\begin{abstract}
A Deep Neural Network (DNN) is a composite function of vector-valued functions, and in order to train a DNN, it is necessary to calculate the gradient of the loss function with respect to all parameters. This calculation can be a non-trivial task because the loss function of a DNN is a composition of several nonlinear functions, each with numerous parameters.
The Backpropagation (BP) algorithm leverages the composite structure of the DNN to efficiently compute the gradient. As a result,  
the number of layers in the
network does not significantly impact the complexity of the calculation.
The objective of this paper is to express the gradient of the loss function in terms of a matrix multiplication using the Jacobian operator. This can be achieved by considering the total derivative of each layer with respect to its parameters and expressing it as a Jacobian matrix. The gradient can then be represented as the matrix product of these Jacobian matrices. This approach is valid because the chain rule can be applied to a composition of vector-valued functions, and the use of Jacobian matrices allows for the incorporation of multiple inputs and outputs. 
By providing concise mathematical justifications, the results can be made understandable and useful to a broad audience from various disciplines.
\end{abstract}

\section{Introduction}
Understanding the process of training a (Deep) Neural Networks (D)NNs, as illustrated in Fig \ref{fig:training}, is not straightforward because it involves many detailed parts. Additionally, modern and sophisticated DNNs are often provided as pre-trained models in Python packages such as Pytorch \cite{paszke2019pytorch} and TensorFlow \cite{abadi2016tensorflow}, which can make it difficult for users to fully understand the training process. 
These packages abstract away many of the implementation details, making it more accessible for users to use these models for their own tasks, but also making it less transparent for the user to understand the inner workings of the model.

\begin{figure}[b]
    \centering 
    \includegraphics[scale=0.3]{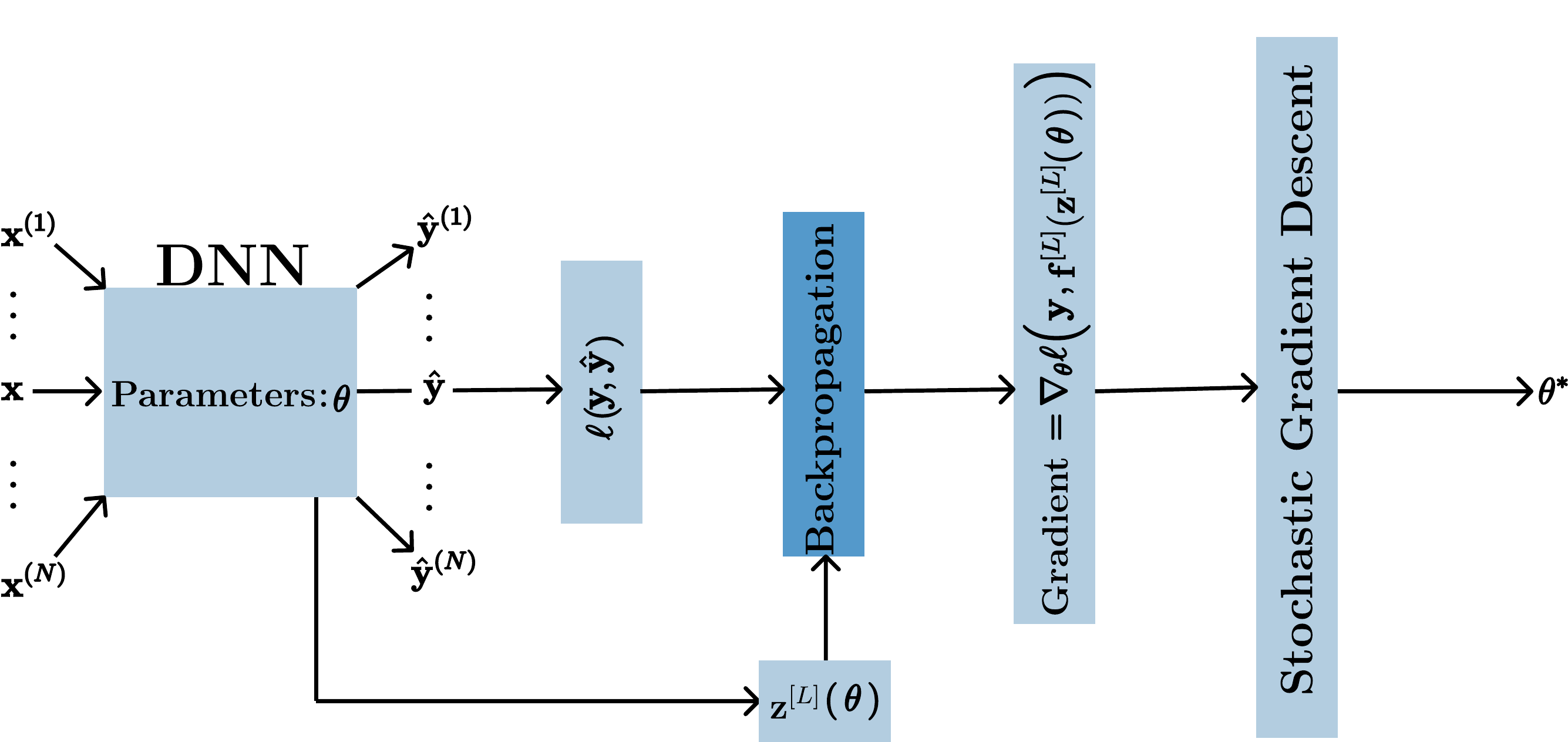}
    \caption{Relationship between training of a NN and BP algorithm}
    \label{fig:training}
\end{figure}


These pre-trained off-the-shelf models can solve variety of tasks such as computer vision or language processing. Convolutional neural networks (CNNs) are currently the most widely used architecture for image classification and other computer vision tasks. Some examples of successful CNN architectures for image classification include ResNet \cite{he2016deep}, Inception \cite{szegedy2015going}, DenseNet \cite{huang2017densely}, and EfficientNet \cite{tan2019efficientnet}. CNNs models are capable of solving image segmentation and object detection which can be done using U-Net \cite{ronneberger2015u} 
and YOLO \cite{redmon2016you}.
For natural language processing tasks, transformer models like BERT \cite{devlin2018bert}, GPT-3 \cite{brown2020language}, and T5 \cite{raffel2020exploring} have achieved state-of-the-art performance on many benchmarks.

Fig \ref{fig:training} shows the training process that utilizes the Stochastic Gradient Descent (SGD) algorithm \cite{robbins1951stochastic} to minimize the loss function of a DNN. As its name suggests, the SGD algorithm requires calculating the gradient\footnote{Please refer Def. \ref{def:gradient} in Appendix for the definition of the gradient of a scalar-valued function.}
of the loss function of a DNN.
All different variants of the SGD algorithm require calculating at least a single gradient associated with a single sample, i.e., $\mathbf{x}$ in Fig. \ref{fig:training}. 
This calculation is a non-trivial task because the loss function of a DNN is a composition of several nonlinear vector-valued functions where each one has numerous parameters. The Backpropagation (BP) algorithm, introduced by \cite{rumelhart1986learning} is an efficient way to calculate the gradient of the loss function of a DNN. This algorithm leverages the composite structure of a DNN to efficiently calculate the gradient of the loss function with respect to the model's parameters, i.e., $\bm{\theta}$ in Fig. \ref{fig:training}. 

\begin{algorithm*}[t]
\caption{The backpropagation algorithm}
\label{alg:backprop}

\begin{algorithmic}[1]

\Require 

Given an $L$-layer DNN or NN with a loss function $\ell$, and a data pair $(\mathbf{x}, \mathbf{y})$. Let $\mathbf{a}^{[0]}:=\mathbf{x}$.
\State

Calculate 
$
\nabla_{\mathbf{z}^{[L]}}
\ell
\big(
\mathbf{y},
\mathbf{f}
(
\mathbf{z}^{[L]}
)\big)$ from Tab. \ref{tab:activaionllosscombination}.
\For{$l=1, \dots, L$}
\If{$l\neq L$}

\begingroup\makeatletter\def\f@size{6}\check@mathfonts
$$
\mathbf{J}_{\mathbf{W}^{[l]}, \mathbf{b}^{[l]}} 
 \big(
\mathbf{z}^{[l]}
\big)
=
\big(
\mathbf{W}^{[L]}
\big)^{\top}
\mathbf{J}_{\mathbf{W}^{[L-1]}, \mathbf{b}^{[L-1]}}
\Big(
\mathbf{f}^{[L-1]}(\mathbf{z}^{[L-1]})
\Big)
\cdots
\mathbf{J}_{\mathbf{W}^{[l]}, \mathbf{b}^{[l]}}
\Big(
\mathbf{f}^{[l]}(\mathbf{z}^{[l]})
\Big)
\begin{bmatrix}
\big(
\mathbf{a}^{[l]}
\big)
^{\top} & 0 & 0 & 
\\
0 & \ddots & 0 & I
\\
0 & 0 & \big(
\mathbf{a}^{[l]}
\big)
^{\top} & 
\end{bmatrix}
$$

\endgroup
\ElsIf{$l = L$}

$$
\mathbf{J}_{\mathbf{W}^{[L]}, \mathbf{b}^{[L]}} 
 \big(
\mathbf{z}^{[L]}
\big)
=
\begin{bmatrix}
\big(
\mathbf{a}^{[L]}
\big)
^{\top} & 0 & 0 & 
\\
0 & \ddots & 0 & \mathbf{I}
\\
0 & 0 & \big(
\mathbf{a}^{[L]}
\big)
^{\top} & 
\end{bmatrix}
$$

\EndIf
\EndFor
\State
Construct 

$$
\begin{aligned}
 \mathbf{J}_{\bm{\theta}}
\mathbf{z}^{[L]}
(\bm{\theta})
&=
\begin{bmatrix}
 \mathbf{J}_{\mathbf{W}^{[1]}, \mathbf{b}^{[1]}} 
 \Big(
\mathbf{z}^{[L]}
(\bm{\theta})
\Big)
\cdots
 \mathbf{J}_{\mathbf{W}^{[L]}, \mathbf{b}^{[L]}} 
 \Big(
\mathbf{z}^{[L]}
(\bm{\theta})
\Big)
\end{bmatrix}
\end{aligned}.
$$

\State Calculate the gradient

$$
\nabla_{\bm{\theta}}
\ell
\big(
\mathbf{y}, \hat{\mathbf{y}}
(\bm{\theta})
\big)
= 
\big(
\mathbf{J}_{\bm{\theta}}
\mathbf{z}^{[L]}(\bm{\theta})
\big)^{\top}
\nabla_{\mathbf{z}^{[L]}}
\ell
\big(
\mathbf{y},
\mathbf{f}
(
\mathbf{z}^{[L]}
)\big)
.
$$

\end{algorithmic}

\end{algorithm*}


In this paper we are going to calculate the gradient of the loss function of a DNN associated with a single sample. The gradient will be derived as the matrix multiplication of Jacobian matrices \footnote{Please refer Def. \ref{def:jacobian} in Appendix for the definition of a Jacobian matrix of a vector-valued function.}. The derivation will be done by considering the total derivative of each layer with respect to its parameters and expressing it as a Jacobian matrix. The gradient can then be represented as the matrix product of these Jacobian matrices. This approach is well-founded because the chain rule is valid for the Jacobian operator. Hence,  the Jacobian operator 
can be applied to a composition of vector-valued functions.
We provide concise mathematical justifications so the results can be made understandable and useful to a broad audience from various disciplines, even those without a deep understanding of the mathematics involved. This is particularly important when communicating complex technical concepts to non-experts, as it allows for a clear and accurate understanding of the results. Additionally, using mathematical notation allows for precise and unambiguous statements of results, which can facilitate replication and further research in the field.

Our results is summarized in Alg. \ref{alg:backprop} for $L$ number of layers. 
As Alg. \ref{alg:backprop} shows the matrix multiplication is done iteratively. The iterative nature comes from the fact that the loss function of a DNN is defined as a composition of $L+1$ functions where the last one $\ell$ is the final function which measures the loss (error) of the prediction and the actual value, i.e., $\ell(\mathbf{y}, \hat{\mathbf{y}})$ in Fig \ref{fig:training} where $\hat{\mathbf{y}}$ is the prediction and $\mathbf{y}$ the actual value.

The algorithm presented in Alg. \ref{alg:backprop} is explained and justified by calculating the gradient of the loss function for networks with one, two, and three layers. These calculations provide insights into the gradients of loss functions for generic neural networks and demonstrate how the gradient of a single-layer network can serve as a model for the last layer of any DNN. Additionally, the calculation of a two-layer network is used to extend the calculation beyond two-layer networks, as seen in the calculation of the gradient of the loss function for LeNet-100-300 \cite{lecun1998gradient}, which is a three-layer network. Finally, we show how convolutional layers can be converted to linear layers in order to calculate their Jacobian matrices. These results can be used to calculate the gradient of the loss function of a CNN.

\section{Notation}

The letters $x$, $\mathbf{x}$, and $\mathbf{W}$ denote a scalar, vector, and matrix, respectively. The letter $\mathbf{I}$ represents the identity matrix. The $i$-th element of a vector $\mathbf{x}$ is denoted by $x_i$. Likewise, $w_{ij}$ denotes the $ij$-th element of a matrix $\mathbf{W}$ located at the $i$-th row and $j$-th column (sometimes written as $w_{i,j}$ for clarity). Also, $\mathbf{W}{i\bullet}$ and $\mathbf{W}{\bullet j}$ denote the $i$-th row and the $j$-th column of the matrix, respectively.
The vector form of a matrix $\mathbf{W}$ is denoted by $\text{Vec}(\mathbf{W})$, where each column of $\mathbf{W}$ is stacked on top of each other, with the first column at the top.
The letter $\mathbf{f}$ is reserved for a vector-valued non-linear activation function of a layer in a DNN (NN), where $\mathbf{z}$ and $\mathbf{a}$ are its input and output, respectively, i.e., $\mathbf{a}=\mathbf{f}(\mathbf{z})$. The letter $\mathbf{x}$ is reserved for the input to a DNN (NN), and $y$ or $\mathbf{y}$ are reserved for the scalar or vector label of the input $\mathbf{x}$. The predictions of a DNN (NN) associated with $y$ or $\mathbf{y}$ are denoted by $\hat{y}$ or $\hat{\mathbf{y}}$, respectively.
Superscripted index inside square brackets denotes the layer of a DNN (NN), e.g., $\mathbf{z}^{[l]}$ is the $\mathbf{z}$ vector corresponding to the $l$-th layer.

\section{Result}

As we have explained earlier, the goal of this paper is to take the first step towards training a DNN, which involves calculating the gradient of a loss function with respect to all parameters of the DNN, i.e., $\nabla_{\bm{\theta}} \ell(\mathbf{y}, \hat{\mathbf{y}}(\bm{\theta}))$, where $\bm{\theta}$ is the vector of parameters, $\hat{\mathbf{y}}$ is the predicted value by the network, $\mathbf{y}$ is the true value, and $\ell(\mathbf{y}, \hat{\mathbf{y}}(\bm{\theta}))$ is the loss incurred to predict the output.
To achieve this goal, two important observations can make the task easier. First, it involves separating the last-layer activation function from the network. Second, it involves using the relationship between the Jacobian operator and the gradient.
\begin{figure}[b]
    \centering 
    \includegraphics[scale=0.3]{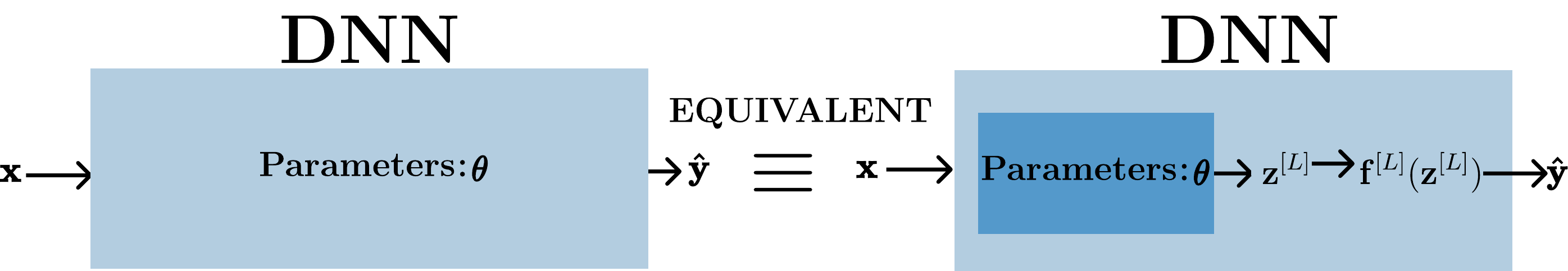}
    \caption{Separating the last layer activation from a DNN}
    \label{fig:separatedLastLayer}
\end{figure}

To fulfill the first step, we observe the following equality:
$$
\ell
\big(
\mathbf{y}, \hat{\mathbf{y}}
(\bm{\theta})
\big)
=
\ell
\Big(
\mathbf{y},
\mathbf{f}^{[L]}
\big(
\mathbf{z}^{[L]}
(\bm{\theta})
\big)
\Big)
$$
which is a consequence of the fact that $\hat{\mathbf{y}}=\mathbf{f}^{[L]}\big(\mathbf{z}^{[L]}(\bm{\theta})\big)$, where $\mathbf{f}^{[L]}$ is the last activation layer and $\mathbf{z}^{[L]}(\bm{\theta})$ is its corresponding input. This equality can be illustrated more clearly as shown in Fig. \ref{fig:separatedLastLayer}.

Second observation utilizes the relationship between the Jacobian operator and the gradient of a scalar-valued function stated in the following lemma.

\begin{lemma}[Jacobian and gradient of a scalar-valued function]
For $f:\mathbb{R}^n\to\mathbb{R}$ as a scalar-valued differentiable function   
$
\nabla_{\mathbf{x}} f(\mathbf{x})
=
\Big(
\mathbf{J}_{\mathbf{x}}
f(\mathbf{x})
\Big)
^{\top}
$
where $\nabla_{\mathbf{x}} f(\mathbf{x})$
is the gradient and $\mathbf{J}_{\mathbf{x}}
f(\mathbf{x})$ is the Jacobian matrix of $f$
at point $\mathbf{x} \in \mathbb{R}^n$ respectively.
\end{lemma}
\begin{proof}
The equality follows from the definitions of a Jacobian matrix as defined in Def. \ref{def:jacobian} and the gradient of a scalar-valued function as defined in Def. in \ref{def:gradient} in Appendix. 
\end{proof}
By using the second observation as $\nabla_{\mathbf{x}} f(\mathbf{x})
=
\Big(
\mathbf{J}_{\mathbf{x}}
f(\mathbf{x})
\Big)
^{\top},$ and making use of the first one as 
$$\ell
\big(
\mathbf{y}, \hat{\mathbf{y}}
(\bm{\theta})
\big)
=
\ell
\Big(
\mathbf{y},
\mathbf{f}^{[L]}
\big(
\mathbf{z}^{[L]}
(\bm{\theta})
\big)
\Big),
$$
one can write the following:
\begin{equation}\label{eq:GradientJacobian}
\begin{aligned}
\nabla_{\bm{\theta}}
\ell
\big(
\mathbf{y}, \hat{\mathbf{y}}
(\bm{\theta})
\big)
&=
\Big(
\mathbf{J}_{\bm{\theta}}
\ell
\big(
\mathbf{y}, \hat{\mathbf{y}}
(\bm{\theta})
\big)
\Big)^{\top}
\\
&=
\Bigg(
\mathbf{J}_{\bm{\theta}}
\ell
\Big(
\mathbf{y},
\mathbf{f}^{[L]}
\big(
\mathbf{z}^{[L]}
(\bm{\theta})
\big)
\Big)
\Bigg)^{\top}
\\
&=
\Bigg(
\mathbf{J}_{\mathbf{z}^{[L]}}
\ell
\Big(
\mathbf{y},
\mathbf{f}^{[L]}
\big(
\mathbf{z}^{[L]}
\big)
\Big)
\mathbf{J}_{\bm{\theta}}
\mathbf{z}^{[L]}
(\bm{\theta})
\Bigg)^{\top}
\\
&=
\Bigg(
\mathbf{J}_{\bm{\theta}}
\mathbf{z}^{[L]}
(\bm{\theta})
\Bigg)^{\top}
\Bigg(
\mathbf{J}_{\mathbf{z}^{[L]}}
\ell
\Big(
\mathbf{y},
\mathbf{f}^{[L]}
\big(
\mathbf{z}^{[L]}
\big)
\Big)
\Bigg)^{\top}
\\
&=
\Bigg(
\mathbf{J}_{\bm{\theta}}
\mathbf{z}^{[L]}
(\bm{\theta})
\Bigg)^{\top}
\nabla_{\mathbf{z}^{[L]}}
\ell
\Big(
\mathbf{y},
\mathbf{f}^{[L]}
\big(
\mathbf{z}^{[L]}
\big)
\Big)
\end{aligned}
\end{equation}
The advantage of Equation (\ref{eq:GradientJacobian}) is that it separates the original gradient calculation into two separate calculations, i.e., $\mathbf{J}_{\bm{\theta}}\mathbf{z}^{[L]}(\bm{\theta})$ and $\nabla_{\mathbf{z}^{[L]}}\ell\Big(\mathbf{y},\mathbf{f}^{[L]}\big(\mathbf{z}^{[L]}\big)\Big)$. 

The calculation of the second term is straightforward because the choice of a loss function and the last activation function in a DNN are not arbitrary. This is illustrated in Fig. \ref{fig:activaionllosscombination}, which shows three common combinations, each associated with a different problem.

\begin{figure}[t]
    \centering 
\includegraphics[scale=0.4]{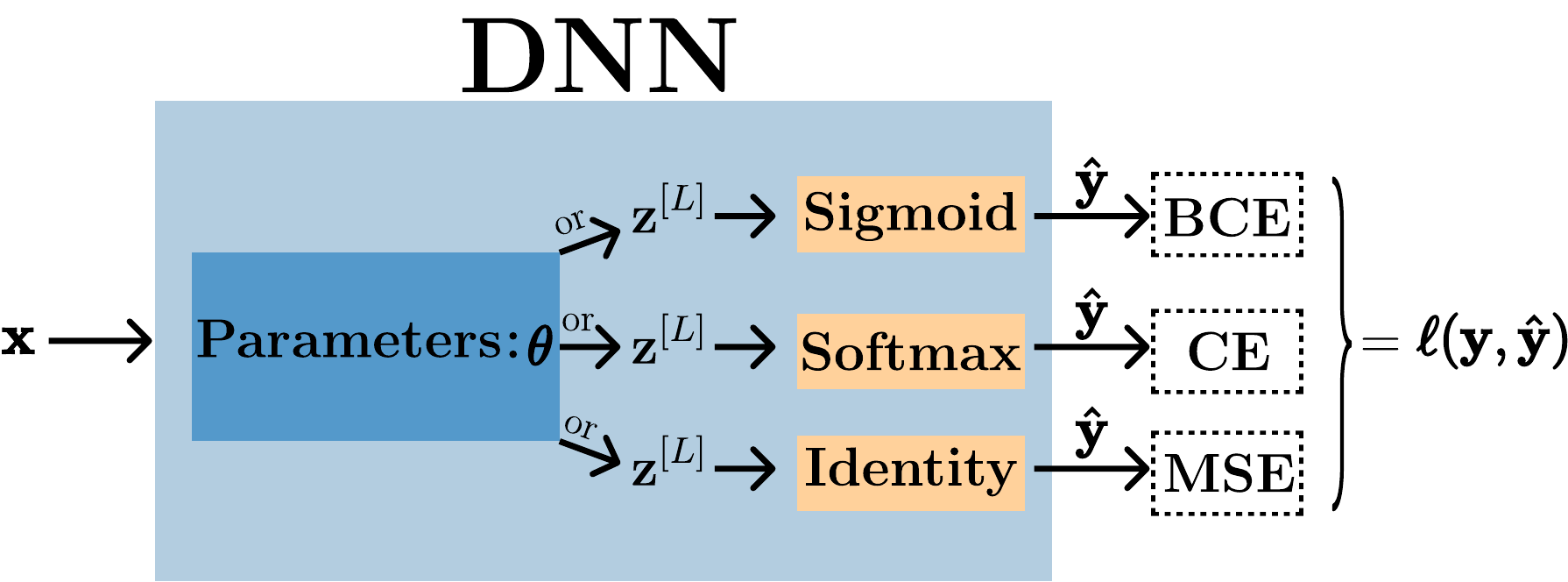}
    \caption{Last-Layer activation function combined with its associated loss}
    \label{fig:activaionllosscombination}
\end{figure}

In a binary classification problem, a sigmoid function is used as the last activation function together with Binary Cross Entropy (BCE) loss, as defined in Definition \ref{def:sigmoid} and Definition \ref{def:BCE} in the Appendix, respectively. Similarly, in a classification problem with more than two classes, a softmax function is used as the last activation function together with Cross Entropy (CE) loss, as defined in Definition \ref{def:softmax} and Definition \ref{def:CE} in the Appendix, respectively. In a regression problem, where the goal is to predict a continuous value, the last-layer activation function is an identity function, and Square Error (SE) is used as the loss function, as defined in Definition \ref{def:SE}. This means that the predicted output $\hat{\mathbf{y}}$ is equal to the final layer's output $\mathbf{z}^{[L]}$. Table \ref{tab:activaionllosscombination} shows the combination of the last activation function with its corresponding loss and the expression for $\nabla_{\mathbf{z}^{[L]}}\ell\Big(\mathbf{y},\mathbf{f}^{[L]}\big(\mathbf{z}^{[L]}\big)\Big)$. For the derivation, please refer to Appendix \ref{append:grad}.

More work is required to calculate the second term, i.e., $\mathbf{J}_{\bm{\theta}}\mathbf{z}^{[L]}(\bm{\theta})$. In the following subsections, we show how to derive this calculation concisely for any number of layers.

\subsection{Gradient of a one-layer network}
We will now focus on computing $\mathbf{J}_{\bm{\theta}}\mathbf{z}^{[L]}(\bm{\theta})$. This calculation can be facilitated by starting with a single layer neural network that is capable of solving a classification problem. The single layer network plays a crucial role in gradient computation as it can be considered as the last layer of any deep neural network (DNN). Due to the presence of only one layer, the superscripts in $\mathbf{z}^{[L]}$ and $\mathbf{f}^{[L]}$ can be omitted, giving us $\mathbf{J}_{\bm{\theta}}\mathbf{z}(\bm{\theta}) := \mathbf{J}_{\bm{\theta}}\mathbf{z}^{[L]}(\bm{\theta})$. The network depicted in Fig. \ref{fig:onelayernetwork} is designed to perform three-class classification based on inputs with four features. Consequently, the weight matrix $\mathbf{W}$ is a $4 \times 3$ and the bias vector $\mathbf{b}$ is a $3 \times 1$, i.e., $\mathbf{W} \in \mathbb{R}^{4 \times 3}$ and $\mathbf{b} \in \mathbb{R}^{3 \times 1}$. As shown in the concise representation of the network in the bottom part of Fig. \ref{fig:onelayernetwork}, 
$\mathbf{z} \in \mathbb{R}^3$, i.e.,
$\mathbf{z}=\big(\mathbf{W}
\big)^{\top}\mathbf{x}
+
\mathbf{b}$. This vector $\mathbf{z}$ is a vector-valued function including $15$ parameters where these parameters are all elements in $\mathbf{W}$ and $\mathbf{b}$, i.e., $15=4\times 3+3\times 1$. According to the definition of Jacobian matrix as defined in Def. \ref{def:jacobian} in Appendix, $\mathbf{J}_{\mathbf{W}, \mathbf{b}} 
\big(
\mathbf{z}(\mathbf{W}, \mathbf{b})
\big)$ is a $3 \times 15$ matrix, i.e., 
$
\mathbf{J}_{\mathbf{W}, \mathbf{b}} 
\big(
\mathbf{z}(\mathbf{W}, \mathbf{b})
\big)
\in \mathbb{R}^{3} \times \mathbb{R}^{15}
$
.

\begin{table*}[h]
\centering
\caption{Last-Layer activation function combined with its associated loss}
\label{tab:activaionllosscombination}
\scalebox{1}{
\begin{tabular}{@{}l *{3}{c}@{}}
\hline\toprule
Loss  & $\ell(\mathbf{y}, \hat{\mathbf{y}})$  
& $
\nabla_{\mathbf{z}^{[L]}}
\ell
\Big(
\mathbf{y},
\hat{\mathbf{y}}
\Big)
=
\nabla_{\mathbf{z}^{[L]}}
\ell
\Big(
\mathbf{y},
\mathbf{f}^{[L]}
\big(
\mathbf{z}^{[L]}
\big)
\Big)
$
 & Proof
\\\midrule
BCE    
& 
$-y \log(\hat{y}) - (1 - y)\log(1 - \hat{y})$
&
$-(y - \hat{y}) \quad (y, \hat{y} \in \mathbb{R})$ 
&
Appendix \ref{subappend:gradBCE}
\\\midrule
CE
& $-\sum_{i=1}^c y_i \log(\hat{y}_i) $
&
$-(\mathbf{y} - \hat{\mathbf{y}}) \quad (\mathbf{y}, \hat{\mathbf{y}}) \in \mathbb{R}^c$
&
Appendix \ref{subappend:gradCE}
\\\midrule
SE 
& 
$\|\mathbf{y}-\hat{\mathbf{y}}\|^2 $     & $-2(\mathbf{y}-\hat{\mathbf{y}}) \quad (\mathbf{y}, \hat{\mathbf{y}} \in \mathbb{R}^m)$  & Appendix \ref{subappend:gradMSE}

\\
\bottomrule\hline
\end{tabular}
}
\end{table*}


For notational simplicity, all the parameters are denoted by  $\bm{\theta}$ which is a vector in $\mathbb{R}^{15}$, and is constructed by the process of vectorization. The  vectorization process stacks each column of $\mathbf{W}$ on top of each other, with the first column on the top to create $\text{Vec}(\mathbf{W})$. Then, the vector of network parameters can be written as $
\bm{\theta}^{\top}:=
\begin{bmatrix}
\big(\text{Vec}(\mathbf{W})\big)^{\top}
& (\mathbf{b})^{\top}
\end{bmatrix}
^{\top}
$.
Therefore to calculate $\mathbf{J}_{\bm{\theta}}\mathbf{z}(\bm{\theta})$ one can write the following:
$$
\begin{aligned}
\mathbf{J}_{\bm{\theta}}\mathbf{z}(\bm{\theta})
&=
\mathbf{J}_{\mathbf{W}, \mathbf{b}}
\Big(
\big(
\mathbf{W}
\big)^{\top}\mathbf{x}
+
\mathbf{b}
\Big)
\\
&=
\mathbf{J}_{\mathbf{W}, \mathbf{b}}
\Big(
\begin{bmatrix}
\big(\mathbf{W}_{\bullet1}
\big)^{\top}
\mathbf{x}
+
b_1
\\
\big(
\mathbf{W}_{\bullet2}
\big)^{\top}
\mathbf{x}
+
b_2
\\
\big(
\mathbf{W}_{\bullet3}
\big)^{\top}
\mathbf{x}
+
b_3
\end{bmatrix}
\Big)
\\
&=
\begin{bmatrix}
\mathbf{x}^{\top} & 0 & 0 & 1 & 0 & 0
\\
0 & \mathbf{x}^{\top} & 0 & 0 & 1 & 0
\\
0 & 0 & \mathbf{x}^{\top} & 0 & 0 & 1
\end{bmatrix}
\\
&=
\begin{bmatrix}
\mathbf{x}^{\top} & 0 & 0 & 
\\
0 & \mathbf{x}^{\top} & 0 & \mathbf{I}_{3\times3}
\\
0 & 0 & \mathbf{x}^{\top} & 
\end{bmatrix}
\in \mathbb{R}^{3\times 18}
\end{aligned}
$$
where $\mathbf{W}_{\bullet i}$ is a column of $\mathbf{W}$ for $i=1,2,3$ and $\mathbf{I}_{3\times 3}$ appears because the derivative of each element of $\mathbf{z}$ with respect to components of $\mathbf{b}$ are either zero or one.
Therefore the gradient of the loss function is calculated as follows
$$
\begin{aligned}
\nabla_{\bm{\theta}}
\ell
\big(
\mathbf{y}, \hat{\mathbf{y}}
(\bm{\theta})
\big)
&= 
\big(
\mathbf{J}_{\bm{\theta}}\mathbf{z}(\bm{\theta})
\big)^{\top}
\nabla_{\mathbf{z}}
\ell
\big(
\mathbf{y},
\mathbf{f}
(
\mathbf{z}
)\big)\\
&=
-
\begin{bmatrix}
 \mathbf{x}   & 0 & 0
 \\
 0 & \mathbf{x} & 0
 \\
 0 & 0 & \mathbf{x}
 \\
 & \mathbf{I}_{3\times3}& 
\end{bmatrix}
(\mathbf{y}-\hat{\mathbf{y}})
\end{aligned}
$$
where the value for 
$
\nabla_{\mathbf{z}}
\ell
\big(
\mathbf{y},
\mathbf{f}
(
\mathbf{z}
)\big)
$
is obtained from Tab. \ref{tab:activaionllosscombination}.
The above gradient is similar to the gradient of one-layer networks whose weights are vectors not matrices as shown in Tab. \ref{tab:onelayernetworks}. As it can be seen from Tab. \ref{tab:onelayernetworks} famous problems such as simple/multiple linear regression, simple binary classification, and logistic regression can be written as a one-layer network whose weight are vectors not matrices.

\begin{figure}[b]
    \centering 
    \includegraphics[scale=0.16]{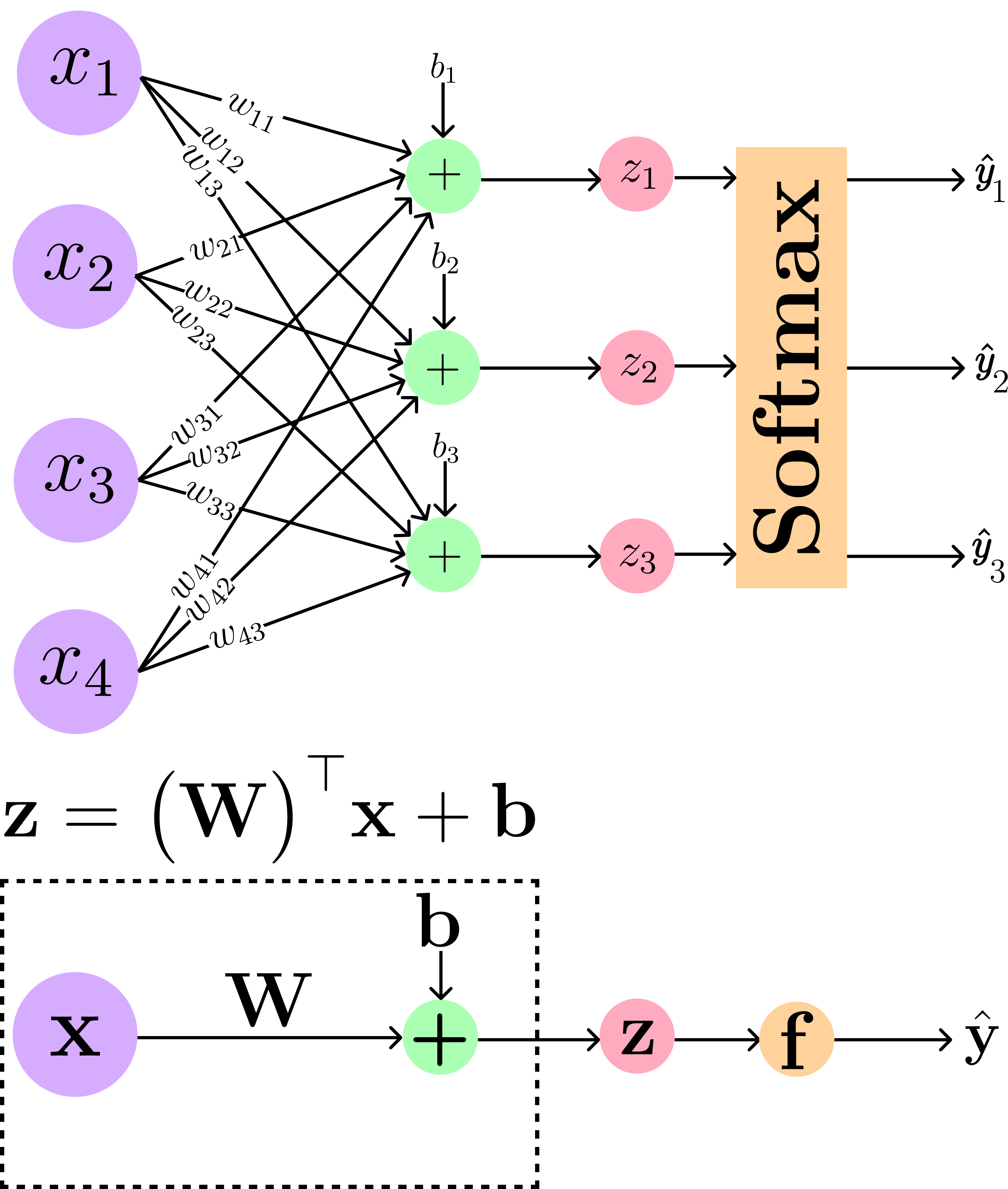}
    \caption{3-class classifier with 4  features and a single layer.}
    \label{fig:onelayernetwork}
\end{figure}
\begin{table*}[h]
\centering
\caption{Famous problems represented as simple neural networks.
}
\label{tab:onelayernetworks}
\scalebox{0.6}{
\begin{tabular}{@{}l *{5}{l}@{}}
\hline
\toprule
 &
 Simple Linear Regression
 &
 Simple Binary Classifier
 &
 Multiple Linear Regression
 &
 Logistic Regression
 \\
\midrule
Architecture
&
\centering
\includegraphics[scale=0.30]{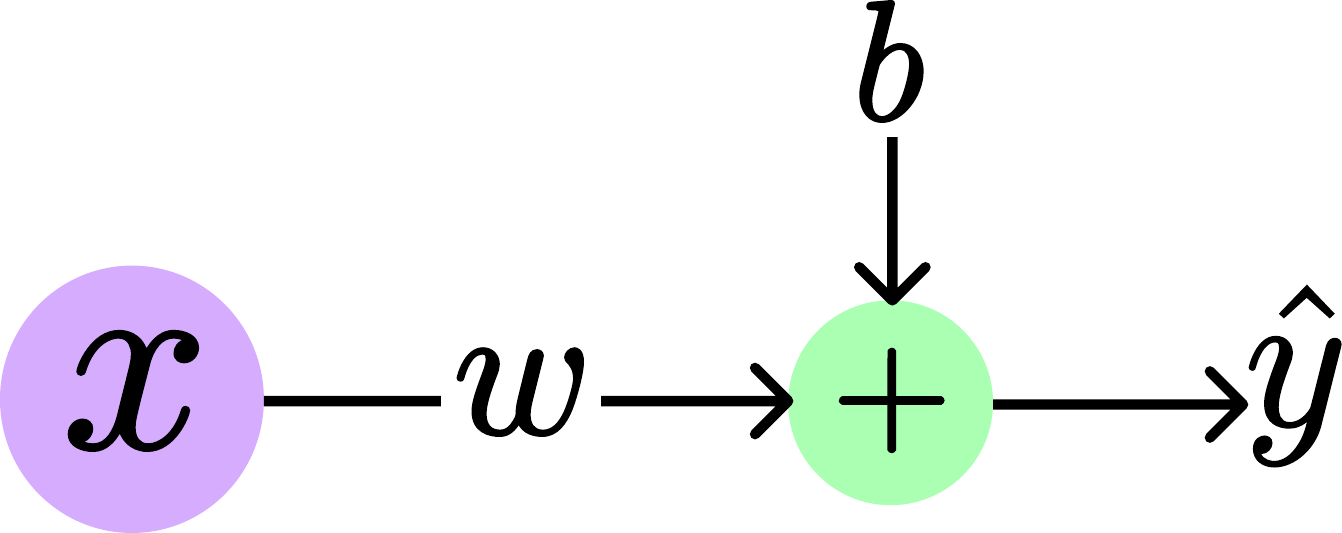}
&
\centering
\includegraphics[scale=0.23]{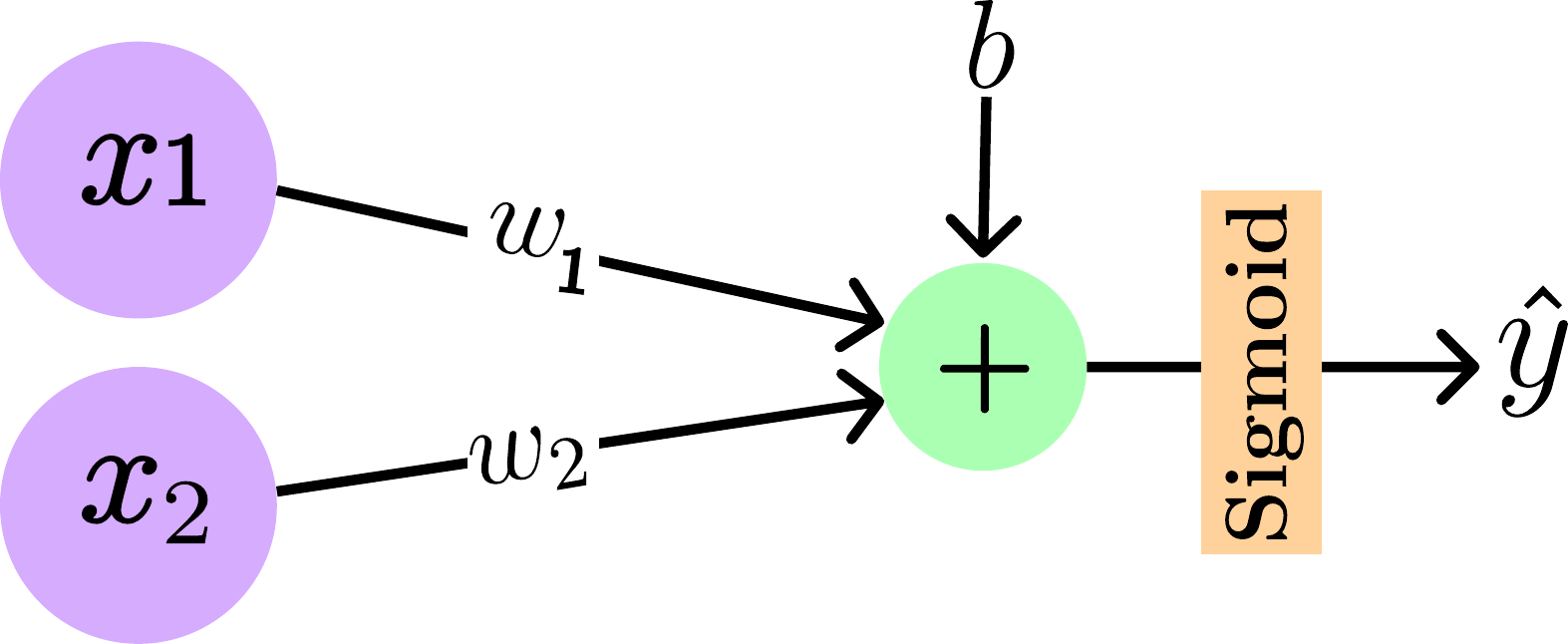}
&
\centering
\includegraphics[scale=0.30]{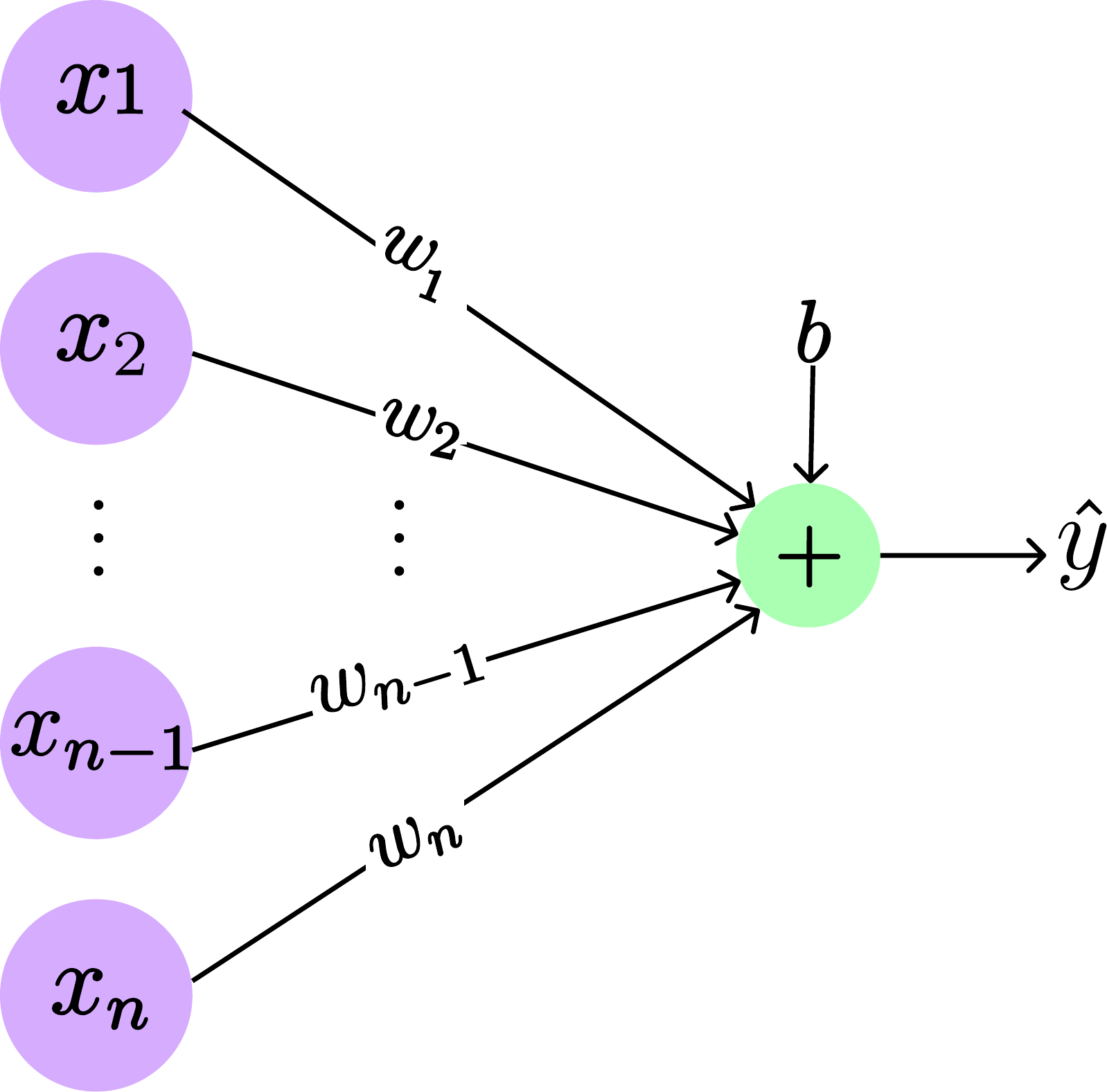}
&
\includegraphics[scale=0.25]{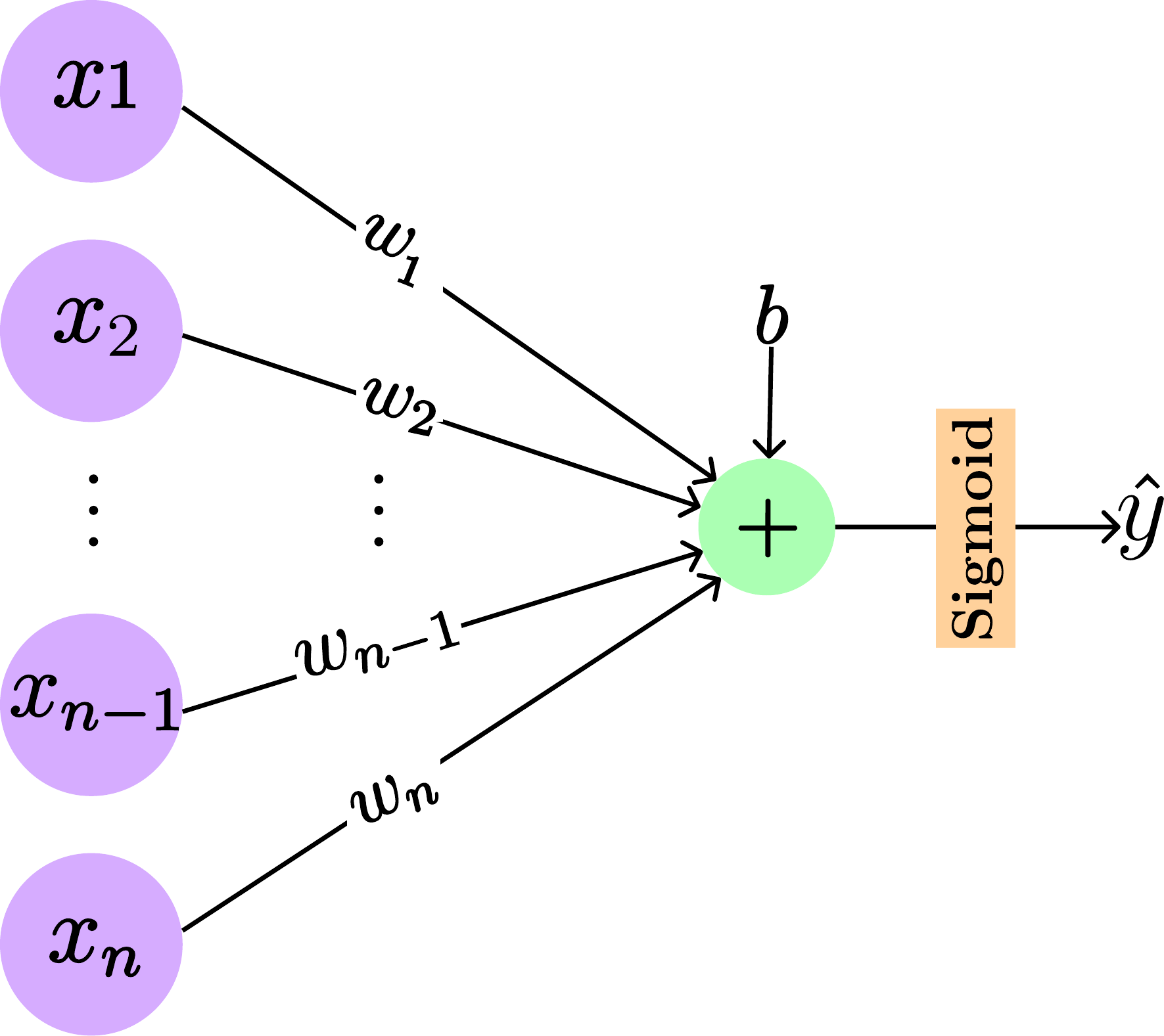}
\\
\midrule
Input
&
$x \in \mathbb{R}$
&
$
\mathbf{x}
=
\begin{bmatrix}
x_1\\x_2    
\end{bmatrix}
\in \mathbb{R}^2
$
&
$
\mathbf{x}
=
\begin{bmatrix}
x_1 \\ x_2 \\ \vdots \\ x_{n-1} \\ x_n    
\end{bmatrix}
\in \mathbb{R}^n
$
&
$
\mathbf{x}
=
\begin{bmatrix}
x_1 \\ x_2 \\ \vdots \\ x_{n-1} \\ x_n    
\end{bmatrix}
\in \mathbb{R}^n
$
\\
\midrule
Parameters
&
$
\bm{\theta}
=\begin{bmatrix}
w\\b    
\end{bmatrix}
\in \mathbb{R}^2
$
&
$
\bm{\theta}
=
\begin{bmatrix}
\mathbf{w}
\\b
\end{bmatrix}
\in \mathbb{R}^3
$
&
$
\bm{\theta}
=\begin{bmatrix}
w_1
\\
\vdots
\\
w_n
\\b
\end{bmatrix}
=\begin{bmatrix}
\mathbf{w}
\\b
\end{bmatrix}
\in \mathbb{R}^{n+1}
$
&
$
\bm{\theta}
=\begin{bmatrix}
w_1
\\
\vdots
\\
w_n
\\b
\end{bmatrix}
=\begin{bmatrix}
\mathbf{w}
\\b
\end{bmatrix}
\in \mathbb{R}^{n+1}
$
\\
\midrule
Predictions 
&
$
\hat{y}=wx+b
$
&
$
\begin{aligned}
\hat{y}
&=
\sigma(\bm{\theta}^{\top}\mathbf{x})
\end{aligned}
$
&
$
\begin{aligned}
\hat{y}
&=
\mathbf{w}^{\top}\mathbf{x}+b
\end{aligned}
$
&
$
\begin{aligned}
\hat{y}
&=
\sigma(\bm{\theta}^{\top}\mathbf{x})
\end{aligned}
$
\\
\midrule
Loss
&
$
\ell=
(y-\hat{y})^2
$
&
$
\ell=
-
y\log(\hat{y})-(1-y)\log(1-\hat{y})
$
&
$
\ell
=
(y-\hat{y})^2 
$
&
$
\ell=
-
y\log(\hat{y})-(1-y)\log(1-\hat{y})
$
\\
\midrule
Gradient
&
$\nabla \ell_{\bm{\theta}}(y, \hat{y})
=
-
2\begin{bmatrix}
x\\1    
\end{bmatrix}
(y-\hat{y})
$
&
$
\begin{aligned}
\nabla \ell_{\bm{\theta}}(y, \hat{y})
&=
-
\begin{bmatrix}
\mathbf{x}\\1    
\end{bmatrix}
(y-\sigma(\bm{\theta}^{\top}\mathbf{x})) 
\end{aligned}
$
&
$\nabla \ell_{\bm{\theta}}(y, \hat{y})
=
-
2\begin{bmatrix}
\mathbf{x}\\1    
\end{bmatrix}
(y-\hat{y})
$
&
$
\begin{aligned}
\nabla \ell_{\bm{\theta}}(y, \hat{y})
&=
-
\begin{bmatrix}
\mathbf{x}\\1    
\end{bmatrix}
(y-\sigma(\bm{\theta}^{\top}\mathbf{x})) 
\end{aligned}
$
\\
\midrule
\bottomrule
\hline
\end{tabular}
}
\end{table*}


Although a one-layer network provides valuable understanding of the relationship between the Jacobian and the gradient of the loss function with respect to $\mathbf{z}$, it is not practical in terms of performance, i.e., accuracy in prediction. To improve performance, adding more layers is recommended. As a result, the following subsection will demonstrate the calculation of the gradient of a two-layer network.

\subsection{Gradient of a two-layer network}
Studying a two-layer network can not only improve performance (accuracy) but also aid in developing a method for calculating the gradient of any deep neural network (DNN) with multiple layers. To demonstrate this, a two-layer network will be considered, as shown in Fig. \ref{fig:twolayernetwork} where block-wise model helps calculating the gradient of its loss.
\begin{figure}[t]
    \centering 
    \includegraphics[scale=0.1]{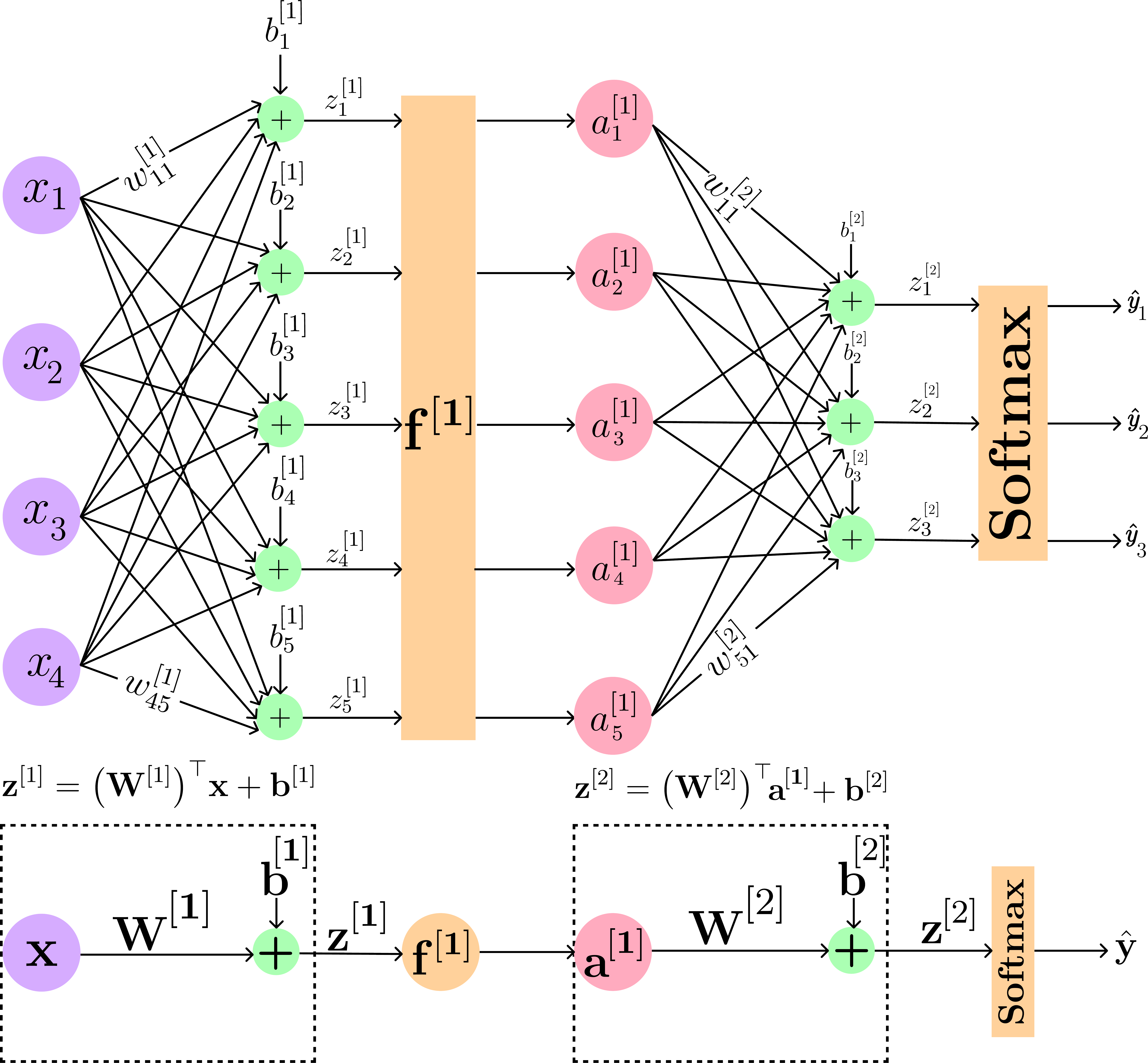}
    \caption{Two-layer network}
    \label{fig:twolayernetwork}
\end{figure}

Similar to a one-layer network, one can write
$
\nabla_{\mathbf{z}^{[2]}}
\ell
\Big(
\mathbf{y},
\mathbf{f}^{[2]}
\big(
\mathbf{z}^{[2]}
\big)\Big)
=
-(\mathbf{y}-\hat{\mathbf{y}})
$
from Tab. \ref{tab:activaionllosscombination}.
To calculate $\mathbf{J}_{\bm{\theta}}
\mathbf{z}^{[2]}
(\bm{\theta})$, observe that it can be separated into parameters of second and first layer as does $\bm{\theta}$. The separation for $\bm{\theta}$ can be written as the following: 
$$
\bm{\theta}^{\top}:=
\begin{bmatrix}
\big(\text{Vec}(\mathbf{W}^{[1]})\big)^{\top}
& (\mathbf{b}^{[1]})^{\top}
&
\big(\text{Vec}(\mathbf{W}^{[2]})\big)^{\top}
& (\mathbf{b}^{[2]})^{\top} 
\end{bmatrix}
^{\top}
$$
which is used to write the below separation of Jacobian matrices:
$$
\begin{aligned}
 \mathbf{J}_{\bm{\theta}}
\mathbf{z}^{[2]}
(\bm{\theta})
&=
\begin{bmatrix}
 \mathbf{J}_{\mathbf{W}^{[1]}, \mathbf{b}^{[1]}} 
 \Big(
\mathbf{z}^{[2]}
(\bm{\theta})
\Big)
&
\mathbf{J}_{\mathbf{W}^{[2]}, \mathbf{b}^{[2]}} 
 \Big(
\mathbf{z}^{[2]}
(\bm{\theta})
\Big)
\end{bmatrix}
\end{aligned}.
$$
Calculating the second term is straight forward because one can write the following:
$$
\begin{aligned}
\mathbf{J}_{\mathbf{W}^{[2]}, \mathbf{b}^{[2]}} 
 \Big(
\mathbf{z}^{[2]}
(\bm{\theta})
\Big) 
&=
\mathbf{J}_{\mathbf{W}^{[2]}, \mathbf{b}^{[2]}}
\Big(
\big(
\mathbf{W}^{[2]}
\big)^{\top}\mathbf{a}^{[1]}
+
\mathbf{b}^{[2]}
\Big)
\\
&=
\begin{bmatrix}
\big(
\mathbf{a}^{[1]}
\big)
^{\top} & 0 & 0 & 
\\
0 & \big(
\mathbf{a}^{[1]}
\big)
^{\top} & 0 & \mathbf{I}_{3\times3}
\\
0 & 0 & \big(
\mathbf{a}^{[1]}
\big)
^{\top} & 
\end{bmatrix}
\end{aligned}
$$
where 
$
\mathbf{a}^{[1]} \in \mathbb{R}^{5}$ because $\mathbf{W}^{[2]} \in \mathbb{R}^{5 \times 3}$.
To calculate the second term 
first observe  that $\mathbf{W}^{[2]}$ is not a function of $\mathbf{W}^{[1]}$ nor $\mathbf{b}^{[1]}$. Then one can write the following:
$$
\begin{aligned}
\mathbf{J}_{\mathbf{W}^{[1]}, \mathbf{b}^{[1]}} 
 \Big(
\mathbf{z}^{[2]}
(\bm{\theta})
\Big)
&=
\mathbf{J}_{\mathbf{W}^{[1]}, \mathbf{b}^{[1]}}
\Big(
\big(
\mathbf{W}^{[2]}
\big)^{\top}\mathbf{a}^{[1]}
+
\mathbf{b}^{[2]}
\Big)
\\
&=
\big(
\mathbf{W}^{[2]}
\big)^{\top}
\mathbf{J}_{\mathbf{W}^{[1]}, \mathbf{b}^{[1]}}
\Big(
\mathbf{a}^{[1]}
\Big)
\\
&=
\big(
\mathbf{W}^{[2]}
\big)^{\top}
\mathbf{J}_{\mathbf{W}^{[1]}, \mathbf{b}^{[1]}}
\Big(
\mathbf{f}^{[1]}(\mathbf{z}^{[1]})
\Big)
\\
&=
\big(
\mathbf{W}^{[2]}
\big)^{\top}
\mathbf{J}_{\mathbf{W}^{[1]}, \mathbf{b}^{[1]}}
\Big(
\mathbf{f}^{[1]}(\mathbf{z}^{[1]})
\Big)
\end{aligned}
$$
The last equality is where the chain rule needs to be used in order to obtain an expression in terms of $\mathbf{W}^{[1]}$ and $\mathbf{b}^{[1]}$, as shown below:
$$
\begin{aligned}
\mathbf{J}_{\mathbf{W}^{[1]}, \mathbf{b}^{[1]}} 
 \Big(
\mathbf{z}^{[2]}
(\bm{\theta})
\Big)
&=
\big(
\mathbf{W}^{[2]}
\big)^{\top}
\mathbf{J}_{\mathbf{z}^{[1]}}
\Big(
\mathbf{f}^{[1]}
\big(
\mathbf{z}^{[1]}
\big)
\Big)
\mathbf{J}_{\mathbf{W}^{[1]}, \mathbf{b}^{[1]}}
\Big(
\big(
\mathbf{W}^{[1]}
\big)^{\top}\mathbf{x}
+
\mathbf{b}^{[1]}
\Big) 
\\
&=
\big(
\mathbf{W}^{[2]}
\big)^{\top}
\mathbf{J}_{\mathbf{z}^{[1]}}
\Big(
\mathbf{f}^{[1]}
\big(
\mathbf{z}^{[1]}
\big)
\Big)
\begin{bmatrix}
\mathbf{x}
^{\top} & 0 & 0
\\
0 & \ddots & 0 & I
\\
0 & 0 & 
\mathbf{x}
^{\top} & 
\end{bmatrix}
\end{aligned}.
$$
Finally, the gradient would be the following:
$$
\begin{aligned}
\nabla_{\bm{\theta}}
\ell
\big(
\mathbf{y}, \hat{\mathbf{y}}
(\bm{\theta})
\big)
&= 
\big(
\mathbf{J}_{\bm{\theta}}
\mathbf{z}^{[2]}(\bm{\theta})
\big)^{\top}
\nabla_{\mathbf{z}^{[2]}}
\ell
\big(
\mathbf{y},
\mathbf{f}
(
\mathbf{z}^{[2]}
)\big)\\
&=
-
\begin{bmatrix}
\begin{bmatrix}
 \mathbf{x}   & 0 & 0
 \\
 0 & \ddots & 0 
 \\
 0 & 0 & \mathbf{x}
 \\
 & \mathbf{I}_{5\times5} &
\end{bmatrix} 
\Big(
\mathbf{J}_{\mathbf{z}^{[1]}}
\Big(
\mathbf{f}^{[1]}
\big(
\mathbf{z}^{[1]}
\big)
\Big)
\Big)^{\top}
\mathbf{W}^{[2]}
\\
\begin{bmatrix}
 \mathbf{a}^{[1]}   & 0 & 0
 \\
 0 &  \mathbf{a}^{[1]} & 0
 \\
 0 & 0 &  \mathbf{a}^{[1]}
 \\
 & \mathbf{I}_{3\times3}& 
\end{bmatrix}
\end{bmatrix}
(\mathbf{y}-\hat{\mathbf{y}})
\end{aligned}
.
$$
The next subsection will demonstrate the calculation of the gradient for a three-layer network, thereby illustrating the extension of a two-layer network gradient to an arbitrary number of layers.
\begin{figure*}[t]
    \centering 
\includegraphics[scale=0.14]{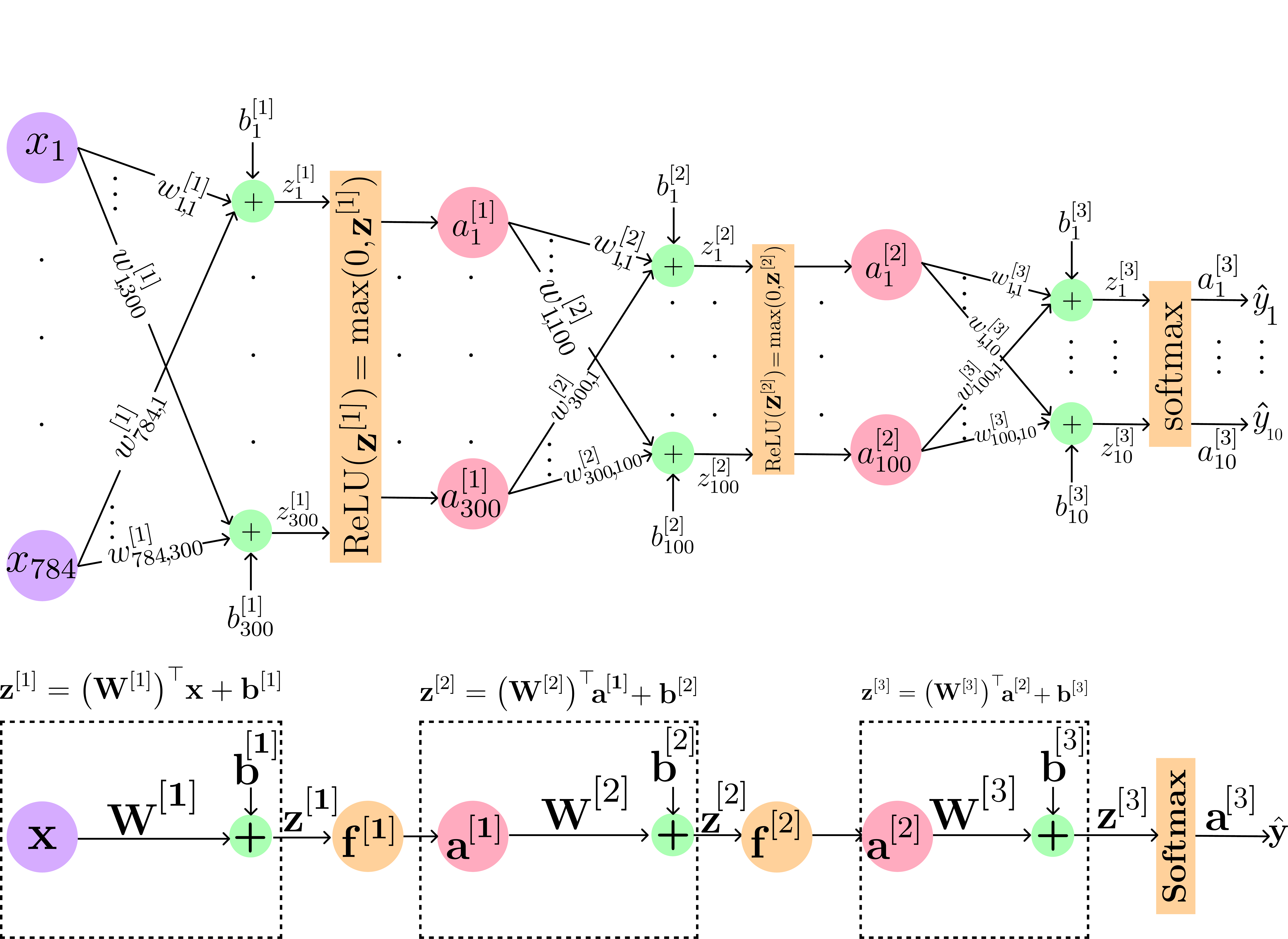}
    \caption{LeNet-100-300-10}
\label{fig:threelayernetwork}
\end{figure*}

\subsection{Gradient of a three-layer network}

In this subsection, we evaluate the gradient of a loss function for a LeNet-100-300-10 network architecture, which consists of 100, 300, and 10 units (neurons) in the first, second and third layer respectively \cite{lecun1998gradient}. This network is used for image classification tasks \cite{lecun1998mnist}. The input to the network is a vectorized representation of $28\times 28$ digit images ($784$ elements) and the output is a vector in $\mathbb{R}^{10}$. The input to the LeNet-100-300-10 is a vector $\mathbf{x}\in \mathbb{R}^{784}$ representing the vectorized image.
Similar to a one- and two-layer networks
$
\nabla_{\mathbf{z}^{[3]}}
\ell
\Big(
\mathbf{y},
\mathbf{f}^{[3]}
\big(
\mathbf{z}^{[3]}
\big)\Big)
=
-(\mathbf{y}-\hat{\mathbf{y}}) \in \mathbb{R}^{10}
$ which is obtained from Tab. \ref{tab:activaionllosscombination}.
To calculate $\mathbf{J}_{\bm{\theta}}
\mathbf{z}^{[3]}
(\bm{\theta})$, three Jacobian are needed to separate out the  parameters of the layers as the following:
$$
\begin{aligned}
 \mathbf{J}_{\bm{\theta}}
\mathbf{z}^{[3]}
(\bm{\theta})
&=
\begin{bmatrix}
\mathbf{J}_{\mathbf{W}^{[1]}, \mathbf{b}^{[1]}} 
 \Big(
\mathbf{z}^{[3]}
(\bm{\theta})
\Big) &\mathbf{J}_{\mathbf{W}^{[2]}, \mathbf{b}^{[2]}}
 \Big(
\mathbf{z}^{[3]}
(\bm{\theta})
\Big)
&
\mathbf{J}_{\mathbf{W}^{[3]}, \mathbf{b}^{[3]}} 
 \Big(
\mathbf{z}^{[3]}
(\bm{\theta})
\Big)
\end{bmatrix}
\end{aligned}.
$$
By following the same steps as for the two-layer network, the gradient can be calculated as:
$$
\nabla_{\bm{\theta}}
\ell
\big(
\mathbf{y}, \hat{\mathbf{y}}
(\bm{\theta})
\big)
= 
\big(
\mathbf{J}_{\bm{\theta}}
\mathbf{z}^{[3]}(\bm{\theta})
\big)^{\top}
\nabla_{\mathbf{z}^{[3]}}
\ell
\big(
\mathbf{y},
\mathbf{f}
(
\mathbf{z}^{[3]}
)\big)
$$
where Jacobian matrices of layers 
$\mathbf{J}^{\top}_{\mathbf{W}^{[1]}, \mathbf{b}^{[1]}}
 \Big( 
\mathbf{z}^{[3]}
(\bm{\theta})
\Big)
$
,
$\mathbf{J}^{\top}_{\mathbf{W}^{[2]}, \mathbf{b}^{[2]}}
 \Big(
\mathbf{z}^{[3]}
(\bm{\theta})
\Big)
$
,
$\mathbf{J}^{\top}_{\mathbf{W}^{[3]}, \mathbf{b}^{[3]}} 
 \Big(
\mathbf{z}^{[3]}
(\bm{\theta})
\Big)$
are:
$$
\begin{bmatrix}
 \mathbf{x} & 0 & 0 
 \\
 0 & \ddots & 0 
 \\
 0 & 0 & \mathbf{x} 
 \\
 & \mathbf{I}_{300\times300} &
\end{bmatrix} 
\Big(
\mathbf{J}_{\mathbf{z}^{[1]}}
\Big(
\mathbf{f}^{[1]}
\big(
\mathbf{z}^{[1]}
\big)
\Big)
\Big)^{\top}
\mathbf{W}^{[2]}
\Big(
\mathbf{J}_{\mathbf{z}^{[2]}}
\Big(
\mathbf{f}^{[2]}
\big(
\mathbf{z}^{[2]}
\big)
\Big)
\Big)^{\top}
\mathbf{W}^{[3]},
$$
$$
\begin{bmatrix}
 \mathbf{a}^{[1]}   & 0 & 0 
 \\
 0 & \ddots & 0 
 \\
 0 & 0 &  \mathbf{a}^{[1]}
 \\
 & \mathbf{I}_{100\times100} &
\end{bmatrix} 
\Big(
\mathbf{J}_{\mathbf{z}^{[2]}}
\Big(
\mathbf{f}^{[2]}
\big(
\mathbf{z}^{[2]}
\big)
\Big)
\Big)^{\top}
\mathbf{W}^{[3]},
$$
$$
\begin{bmatrix}
 \mathbf{a}^{[2]}   & 0 & 0
 \\
 0 &  \ddots & 0
 \\
 0 & 0 &  \mathbf{a}^{[2]}
 \\
 & \mathbf{I}_{10\times10}& 
\end{bmatrix},
$$
and 
$\nabla_{\mathbf{z}^{[3]}}
\ell
\big(
\mathbf{y},
\mathbf{f}
(
\mathbf{z}^{[3]}
)\big)
=-(\mathbf{y}-\hat{\mathbf{y}})
$
which is consistent with Alg. \ref{alg:backprop}.

\subsection{Jacobian of activation functions}
In this subsection we elaborate on
$\mathbf{J}_{\mathbf{z}^{[l]}}
\Big(
\mathbf{f}^{[l]}
\big(
\mathbf{z}^{[l]}
\big)
\Big)
$ where $\mathbf{f}^{[l]}$ is the $l$-th activation layer whose corresponding input is $\mathbf{z}^{[l]}$ for $l=1,\dots, L-1$. 
When $l \neq L$, the vector $\mathbf{f}^{[l]}(\mathbf{z}^{[l]})$ is typically obtained by applying a single univariate function $f$ to each element of $\mathbf{z}^{[l]}$. The most common activation function used in DNNs is the Rectified Linear Unit (ReLU) function, defined as $f(x)=\max(0,x)$ \cite{fukushima1969visual,fukushima1975cognitron,rumelhart1986general,nair2010rectified}. Algebraically, this operation can be represented as
$$
\mathbf{f}^{\top}
\big(
\mathbf{z}
\big)
=
\begin{bmatrix}
f(z_1) &
\cdots &
f(z_{d})
\end{bmatrix}
^{\top}
$$
where $f$ is the same univariate function applied to all elements of $\mathbf{z}$, $d$ is the output size of the layer, and we have omitted superscripts for clarity.
This special structure results in a diagonal matrix, i.e., 
$$
\mathbf{J}_{\mathbf{z}}
\big(
\mathbf{f}
(
\mathbf{z}
)
\big)
=
\text{diag}
\Big(
f'(z_1),
\cdots,
f'(z_{d})
\Big)
$$
where $f'$ is the derivative of a univariate function. 
For the special case of the ReLU function 
$
\mathbf{J}_{\mathbf{z}}
\big(
\mathbf{f}
(
\mathbf{z}
)
\big)
$
is a diagonal matrix of zeros and ones which are associated to the negative and positive elements of $\mathbf{z}$.
Multiplying such a matrix from the left to any matrix $\mathbf{W}$ results in removing the rows of $\mathbf{W}$ associated to zero elements in $
\mathbf{J}_{\mathbf{z}}
\big(
\mathbf{f}
(
\mathbf{z}
)
\big)
$
which greatly decreases the computation. The zero-th norm of the parameter vector can be minimized in sparse optimizations using these intuitions, as noted in \cite{damadi2022gradient}.

Note that the derivations so far have only considered fully connected networks. However, for computer vision tasks, CNN models are utilized, which employ convolution operations instead of matrix multiplication in some layers. In the next subsection, we will demonstrate how a convolutional layer can be transformed into a fully connected network.

\subsection{Convolution as matrix multiplication}
A 2-D convolution operation is a mathematical operation that is used to extract features or patterns from a 2-dimensional input (matrix), such as an image. It works by applying a filter or kernel, which is also a matrix, to the input image \cite{lecun1998gradient}. The filter is moved across the image, performing element-wise multiplications with the overlapping regions of the image and filter, and then summing the results. This process is repeated for every position of the filter on the image, resulting in a new matrix output, known as a feature map. Fig. \ref{fig:convolution} shows the process of a convolution operation where the 3x3 and 2x2 matrices represent the input and the filter respectively. As Fig. \ref{fig:convolution} illustrates the filter slides over the input image, one pixel at a time, and performs element-wise multiplications with the overlapping region of the image. The result of these multiplications is then summed, and the sum is stored in the corresponding location of the output feature map. 

\begin{figure}[t]
    \centering \includegraphics[scale=0.2]{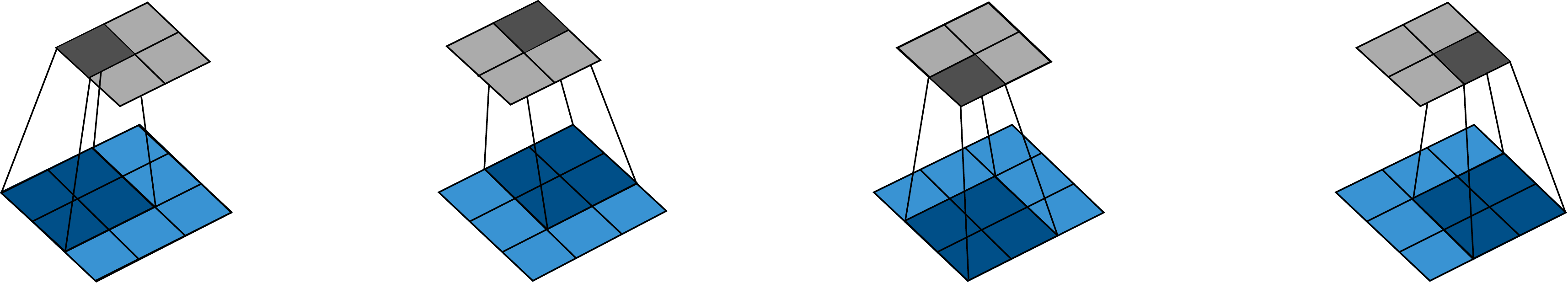}
    \caption{An illustration of convolution}
\label{fig:convolution}
\end{figure}

The size of the filter and the stride (the number of pixels the filter is moved each time) determine the size of the output feature map. Additionally, the filter can be applied multiple times with different filter parameters, to extract different features from the same input image.

\begin{lemma}
The convolution operation between two matrices, $\mathbf{X}$ and $\mathbf{K}$, can be represented as a matrix multiplication. Specifically, it can be represented as the product of a Toeplitz matrix (or diagonal-constant matrix) of $\mathbf{K}$ and the vector obtained from stacking the columns of the transpose of $\mathbf{X}$ in the order of the first one on top. Mathematically, this can be represented as:
$$\mathbf{X}*\mathbf{K}=\text{Teop}(\mathbf{K})\text{Vec}(\mathbf{X}^{\top})$$
where $\mathbf{X}\in \mathbb{R}^{m_X \times n_X}$ and $\mathbf{K}\in \mathbb{R}^{m_K \times n_K}$ and $m_X, n_X, m_X, n_X \in \mathbb{N}$.
  
\end{lemma}
\begin{remark}
The above representation allows for the convolution operation to be computed efficiently using matrix multiplication, which can be parallelized and accelerated on a GPU. 
\end{remark}

To clarify the above lemma, consider $\mathbf{X}\in \mathbb{R}^{3 \times 3}$ and $\mathbf{K}\in \mathbb{R}^{2 \times 2}$ as the input and filter respectively. Then, one can verify the lemma by writing the following:
$$
\mathbf{X}*\mathbf{K}=
\begin{bmatrix}
k_{1} & k_{2} & 0 & k_{3} & k_{4} & 0 & 0 & 0 & 0\\
0 & k_{1} & k_{2} & 0 & k_{3} & k_{4} & 0 & 0 & 0  \\
0 & 0 & k_{1} & k_{2} & 0 & k_{3} & k_{4} & 0 & 0  \\
0 & 0 & 0 & k_{1} & k_{2} & 0 & k_{3} & k_{4} & 0
\end{bmatrix}
\begin{bmatrix}
x_{1} \\ x_{2} \\ x_{3} \\ x_{4} \\ x_{5} \\ x_{6} \\ x_{7} \\ x_{8} \\ x_{9}
\end{bmatrix}
$$
\begin{lemma}
For an input $\mathbf{X}$ to a convolutional layer that has $r$ number of filters, each calculation 
$\mathbf{X}*\mathbf{K}_i + \mathbf{B}_i$ is equivalent to
$
(\mathbf{W}_i)^{\top}\mathbf{x} + \mathbf{b}_i
$
where 
$\mathbf{K}_i$ is the matrix of the $i$-th filter, $\mathbf{B}_i$ is the matrix of associated bias for each filter, $\mathbf{W}_i
=
\Big(
\text{Teop}(\mathbf{K}_i)
\Big)^{\top}
$,  $\mathbf{b}_i=\text{Vec}(\mathbf{B}_i)$
for $i=1, \dots, r$, and $\mathbf{x}:=\text{Vec}(\mathbf{X}^{\top})$.
\end{lemma}
According to the above lemma, we can convert a convolutional neural network to a typical fully connected one and find its gradient.

\section{Conclusion}
In this paper, we demonstrated the utilization of the Jacobian operator to simplify the gradient calculation process in DNNs. We presented a matrix multiplication-based algorithm that expresses the BP algorithm using Jacobian matrices and applied it to determine gradients for single, double, and three-layer networks. Our calculations offered insights into the gradients of loss functions in DNNs; for instance, the gradient of a single-layer network can serve as a model for the final layer of any DNN. We also expanded our findings to cover more intricate architectures such as LeNet-100-300-10 and demonstrated that the gradients of convolutional neural network layers can be transformed into linear layers. These results can aid research on compressing DNNs that utilize the full gradient, as noted in \cite{damadi2022amenable}. Furthermore, they can benefit sparse optimization in both deterministic and stochastic settings, where the Iterative Hard Thresholding (IHT) algorithm uses the full gradient for a sparse solution in deterministic settings \cite{damadi2022gradient} and the mini-batch Stochastic IHT algorithm is employed in the stochastic context \cite{damadi2022convergence}.
We provided concise mathematical justifications to make the results clear and useful for people from different fields, even those without a deep understanding of the involved mathematics. This was particularly important when communicating complex technical concepts to non-experts as it allowed for a clear and accurate understanding of the results. Additionally, using mathematical notation allowed for precise and unambiguous statements of results, facilitating replication and further research in the field.
As next steps, we intend to study the calculation of gradients for loss functions in various types of neural networks such as residual, recurrent, Long Short-Term Memory (LSTM), and Transformer networks. We will also explore the Jacobian of batch normalization to further our understanding of the method.

\newpage

\bibliographystyle{IJCNN2023_conference}
\bibliography{IJCNN2023_conference}

\newpage

\appendix

\section{Operators}

\begin{definition}[Jacobian matrix of a vector-valued function]\label{def:jacobian}
Let $\mathbf{f}: \mathbb{R}^n \to \mathbb{R}^m$ be a differentiable vector-valued function where 
$
\mathbf{f}(\mathbf{x})
=
\begin{bmatrix}
f_1(\mathbf{x})
\\
\vdots
\\
f_m(\mathbf{x})
\end{bmatrix}
$
 and $\mathbf{x} \in \mathbb{R}^n$. This function takes a point $\mathbf{x} \in \mathbb{R}^n$ as an input and produces $\mathbf{f}(\mathbf{x})\in \mathbb{R}^m$ as the output. The Jacobian matrix of $\mathbf{f}$ with respect to $\mathbf{x}$ is defined to be an $m \times n$ matrix denoted by $\mathbf{J}_{\mathbf{x}} \mathbf{f}(\mathbf{x})$ as the following:
 $$
 \mathbf{J}_{\mathbf{x}} \mathbf{f}(\mathbf{x})
 =
 \begin{bmatrix}
 \frac{\partial f_1(\mathbf{x})}{\partial x_1}
 &
 \cdots
&
 \frac{\partial f_1(\mathbf{x})}{\partial x_n}
 \\
 \vdots
 &
 \ddots
 &
 \vdots
 \\
 \frac{\partial f_m(\mathbf{x})}{\partial x_1}
 &
 \cdots
&
 \frac{\partial f_m(\mathbf{x})}{\partial x_n}
 \end{bmatrix}
 =
 \begin{bmatrix}
 \big(
\nabla_{\mathbf{x}} f_1(\mathbf{x}) 
\big)^{\top} 
\\
\vdots
\\
\big(
\nabla_{\mathbf{x}}
f_m(\mathbf{x})  
\big)^{\top}
 \end{bmatrix}.
 $$
\end{definition}

\begin{definition}[Gradient of a scalar-valued function]\label{def:gradient}
Let $f:\mathbb{R} \to \mathbb{R}$ be a differentiable scalar-valued function. The gradient of  $\nabla f: \mathbb{R}^n \to \mathbb{R}^n$ at $\mathbf{x} \in \mathbb{R}^n$ is defined as the following:
$$
\nabla f(\mathbf{x})
=
\begin{bmatrix}
 \frac{\partial f(\mathbf{x})}{\partial x_1}
\\
 \vdots
\\
 \frac{\partial f(\mathbf{x})}{\partial x_n}
 \end{bmatrix}.
$$
\end{definition}

\section{Activation functions}
\begin{definition}[Sigmoid function]\label{def:sigmoid}
A sigmoid function $\sigma: \mathbb{R} \to [0,1]$ is defined as $\sigma(x)=\frac{e^x}{1+e^x}$.
\end{definition}

\begin{definition}[Softmax function]\label{def:softmax}
A softmax function $\bm{\sigma}: \mathbb{R}^c \to \mathbb{R}^c$ is defined for $c\geq 3$ as the following:
$$
\bm{\sigma}(\mathbf{x})
=
\begin{bmatrix}
\frac{e^{x_1}}{\sum_{j=1}^{c}e^{x_j}}
\\
\vdots
\\
\frac{e^{x_c}}{\sum_{j=1}^{c}e^{x_j}}
\end{bmatrix}
\in \mathbb{R}^c.
$$
\end{definition}
\section{Loss functions}
\begin{definition}[Binary Cross Entropy Loss]\label{def:BCE}
Let $\hat{y} \in [0,1]$ be a predicted probability for a true label whose value is either zero or one, i.e., $y \in \{0,1\}$. The Binary Cross Entropy (BCE) loss is defined as follows:
$$
\text{BCE}(y,\hat{y})
=
-\Big(y \log(\hat{y}) + (1 - y)\log(1 - \hat{y})
\Big).
$$
\end{definition}
\begin{remark}[BCE loss]
The BCE loss is minimized when the predicted label is close to the true label.
The BCE loss has a smooth and continuous gradient which makes it suitable for use with gradient-based optimization algorithms. It is also a convex function with respect to the variable $\hat{y}$.
\end{remark}
\begin{definition}[Cross Entropy Loss]\label{def:CE}
Let $\hat{\mathbf{y}} \in (0,1)^c \in \mathbb{R}^c$ ($c \geq 3$) be a predicted probability for a true one-hot vector label $\mathbf{y} \in \mathbb{R}^c$, i.e., $y_j=1, y_i=0$ for $j\neq i=1, \dots, c$. The Cross Entropy (CE) loss is defined as follows:
$$
\text{CE}(y,\hat{y})
=
-\sum_{i=1}^{c}y_i\log (\hat{y}_i).
$$
\end{definition}
\begin{remark}
Cross entropy function is commonly used in machine learning and information theory to measure the difference between two probability distributions. It is often used as a loss function to evaluate the performance of classification models.

In the context of machine learning, cross entropy is typically used to measure the difference between the predicted probability distribution (outputted by the model) and the true probability distribution (which represents the actual labels of the data). The cross entropy loss function is designed to penalize the model when it assigns low probabilities to the true labels, and to reward the model when it assigns high probabilities to the true labels.
\end{remark}
\begin{definition}[Squared Error Loss]\label{def:SE}
Let $\hat{y} \in \mathbb{R}$ be a predicted value for a true label whose value is $y \in \mathbb{R}$. The Square Error (SE) loss is defined as follows:
$$
\text{SE}(y,\hat{y})
=
(y-\hat{y})^2.
$$
\end{definition}

\section{Gradient of the loss with respect to $\mathbf{z}^{[L]}$}\label{append:grad}

\subsection{Sigmoid with BCE}\label{subappend:gradBCE}
\begin{lemma}[Gradient of BCE loss]
Let $\bm{x}^{(i)}$ be an input to a one-layer network solving binary classification with $y^{(i)} \in \{0,1\}$ be its true label and $\hat{y}^{(i)}=\sigma
\left(\mathbf{w}^{\top}\bm{x}^{(i)}+b
\right)$ is the predicted probability corresponding to the input for $i=1,\dots,N$.
The gradient of the BCE loss is given as follows:

$$
\nabla_{\bm{\theta}}
\text{BCE}
\left(
y^{(i)}, \sigma
\left(\mathbf{w}^{\top}\bm{x}^{(i)}+b
\right)
\right) 
=
-
\left(
y^{(i)} 
-
\sigma
\left(
\mathbf{w}^{\top}\bm{x}^{(i)}+b
\right)
\right)
\begin{bmatrix}
\bm{x}^{(i)}
\\
1
\end{bmatrix}
$$
for $i=1,\dots,N$.
\end{lemma}
\begin{proof}
We calculate the following for a fixed $i \in \{1, \dots, N\}$:
$$
\begin{aligned}   
\nabla_{\bm{\theta}}
\Bigg(
\text{BCE}
\left(
y^{(i)}, \sigma
\left(\mathbf{w}^{\top}\bm{x}^{(i)}+b
\right)
\right)
\Bigg)
&= -
 \nabla_{\bm{\theta}}
 \Bigg(
 y^{(i)} \log
 \left(\sigma
\left(\mathbf{w}^{\top}\bm{x}^{(i)}+b
\right)
\right) 
\\ &+ (1 - y^{(i)})\log
 \left(
 1 - \sigma
\left(\mathbf{w}^{\top}\bm{x}^{(i)}+b
\right)
\right)
\Bigg) 
.
\end{aligned}
$$
To calculate the above observe the following:
$$
\begin{aligned}
\nabla_{\bm{\theta}}
 \left(
 \sigma
\left(\mathbf{w}^{\top}\bm{x}^{(i)}+b
\right)
\right)
=
\begin{bmatrix}
\bm{x}^{(i)}
\\
1
\end{bmatrix}\sigma
\left(\mathbf{w}^{\top}\bm{x}^{(i)}+b
\right)
 \left(
 1 - \sigma
\left(\mathbf{w}^{\top}\bm{x}^{(i)}+b
\right)
\right).
\end{aligned}
$$
Hence, we can write the following:
$$
\begin{aligned}
\nabla_{\bm{\theta}}
\ell
\big(
\mathbf{y}, \hat{\mathbf{y}}
(\bm{\theta})
\big)
&=
\nabla_{\bm{\theta}}
\Bigg(
\text{BCE}
\left(
y^{(i)}, \sigma
\left(\mathbf{w}^{\top}\bm{x}^{(i)}+b
\right)
\right)
\Bigg)
\\ 
&=-\Bigg(
y^{(i)} 
\frac{
\sigma
\left(
\mathbf{w}^{\top}\bm{x}^{(i)}+b
\right)
\left(
 1 - \sigma
\left(\mathbf{w}^{\top}\bm{x}^{(i)}+b
\right)
\right)
}{\sigma
\left(\mathbf{w}^{\top}\bm{x}^{(i)}+b
\right)}
\begin{bmatrix}
\bm{x}^{(i)}
\\
1
\end{bmatrix}
\\
 &- 
 (1 - y^{(i)})
 \frac{
\sigma
\left(
\mathbf{w}^{\top}\bm{x}^{(i)}+b
\right)
\left(
 1 - \sigma
\left(\mathbf{w}^{\top}\bm{x}^{(i)}+b
\right)
\right)
}{
1-\sigma
\left(\mathbf{w}^{\top}\bm{x}^{(i)}+b
\right)}
\begin{bmatrix}
\bm{x}^{(i)}
\\
1
\end{bmatrix}
\Bigg) 
\\
&=
-
y^{(i)} 
\left(
 1 - \sigma
\left(\mathbf{w}^{\top}\bm{x}^{(i)}+b
\right)
\right)
\begin{bmatrix}
\bm{x}^{(i)}
\\
1
\end{bmatrix}
\\
&+
 (1 - y^{(i)})
\sigma
\left(
\mathbf{w}^{\top}\bm{x}^{(i)}+b
\right)
\begin{bmatrix}
\bm{x}^{(i)}
\\
1
\end{bmatrix}
\\
&=
-
\left(
y^{(i)} 
-
\sigma
\left(
\mathbf{w}^{\top}\bm{x}^{(i)}+b
\right)
\right)\begin{bmatrix}
\bm{x}^{(i)}
\\
1
\end{bmatrix}
\end{aligned}
$$
\end{proof}

\subsection{Softamx with CE}\label{subappend:gradCE}
\begin{lemma}[Gradient of CE loss w.r.t $\mathbf{z}$]
Let $\hat{\mathbf{y}} \in (0,1)^c \in \mathbb{R}^c$ ($c \geq 3$) be a predicted probability such that $\hat{\mathbf{y}}=\bm{\sigma}(\mathbf{z})$ where $\bm{\sigma}$ is a softmax function as defined in Def. \ref{def:softmax} and $\mathbf{z} \in \mathbb{R}^n$. And, let $\mathbf{y}$ be a true one-hot vector label $\mathbf{y} \in \mathbb{R}^c$, i.e., $y_j=1, y_i=0$ for $j\neq i=1, \dots, c$. 
The gradient of the CE loss with respect to $\mathbf{z}$ is the following:
$$
\nabla_{\mathbf{z}}
\ell
\left(
\mathbf{y},
\hat{\mathbf{y}}(\mathbf{z})
\right)
=
\nabla_{\mathbf{z}}
\text{CE}
\left(
\mathbf{y},
\hat{\mathbf{y}}(\mathbf{z})
\right)
=
-(\mathbf{y} - \hat{\mathbf{y}}).
$$
\begin{proof}
To calculate $\nabla_{\mathbf{z}}
\text{CE}
\left(
\mathbf{y},
\hat{\mathbf{y}}(\mathbf{z})
\right)$ one can rewrite it as $\nabla_{\mathbf{z}}
\text{CE}
\left(
\mathbf{y},
\bm{\sigma}(\mathbf{z})
\right)$. Then the loss can be expanded as follows:
$$
\text{CE}
\left(
\mathbf{y},
\bm{\sigma}(\mathbf{z})
\right)
=
-
\sum_{i=1}^c
y_i
\log 
\left(
\frac{e^{z_i}}{\sum_{j=1}^c e^{z_j}}
\right)
$$

each component can be calculated as the following:  
$$
\begin{aligned}
\frac{\partial}{\partial z_j}
\left(
-
\sum_{i=1}^c
y_i
\left(
z_i-\log (\sum_{j=1}^c e^{z_j})
\right)
\right) 
&=
\frac{\partial}{\partial z_j}
\left(
-
\sum_{i=1}^c(y_iz_i)
+
\sum_{i=1}^c
y_i
\left(
\log (\sum_{j=1}^c e^{z_j})
\right)
\right)
\\
&=
-y_j+
\sum_{i=1}^c
y_i
\frac{e^{z_j}}{\sum_{j=1}^c e^{z_j}}
\\
&=
-y_j+
\sum_{i=1}^c
y_i
(\bm{\sigma}(\mathbf{z}))_j
\\
&=
-y_j+
(\bm{\sigma}(\mathbf{z}))_j
\sum_{i=1}^c
y_i
\\
&=
-y_j+
(\bm{\sigma}(\mathbf{z}))_j
\end{aligned}
$$
where in the last equality we have use the fact that $\sum_{i=1}^c
y_i=1$. Since the above calculation is for the $j$-th component and $(\hat{\mathbf{y}})_j=(\bm{\sigma}(\mathbf{z}))_j$, by calculating other components we get the desired result.
\end{proof}

\end{lemma}
\subsection{Square Error}\label{subappend:gradMSE}
\begin{lemma}
Let $\bm{x}^{(i)}$ be an input to a one-layer network with 
$y^{(i)}\in \mathbb{R}$ be the corresponding true value and
$$ \hat{y}^{(i)} =\mathbf{w}^{\top}\mathbf{x}^{(i)}+ b \quad \hat{y}^{(i)} in \mathbb{R} $$ be the corresponding prediction  for $i=1,\dots,N$. The gradient of SE loss is defined as follows:
$$
\begin{aligned}
\nabla_{\bm{\theta}}
\text{SE}(y^{(i)},\mathbf{w}^{\top}\mathbf{x}^{(i)}+ b)
&=
\nabla_{\bm{\theta}}
(y^{(i)}-(\mathbf{w}^{\top}\mathbf{x}^{(i)}+ b))^2
\\
&=
-
2\begin{bmatrix}
\mathbf{x}^{(i)}\\1    \end{bmatrix}
(y^{(i)}-\hat{y}^{(i)})    
\end{aligned}
$$
for $i=1,\dots,N$.
\end{lemma}
\begin{proof}
Let $\bm{\theta}:=[\mathbf{w}^{\top} b]^{\top}$.
Then,
$$
\nabla_{\bm{\theta}}
(y^{(i)}-\hat{y}^{(i)})^2
=
-
2(y^{(i)}-\hat{y}^{(i)})
\nabla_{\bm{\theta}}
\hat{y}^{(i)}
=
-
2\begin{bmatrix}
\mathbf{x}^{(i)}\\1    
\end{bmatrix}
(y^{(i)}-\hat{y}^{(i)})
.$$
\end{proof}

\end{document}